\documentclass{svproc}
\usepackage{url}

\usepackage{amsmath}
\usepackage{amsfonts}
\usepackage{amssymb,amsbsy}
\usepackage{amsmath}
\usepackage{amsfonts, amscd}
\usepackage{algorithm, algorithmic}
\usepackage{color}
\usepackage{url}
\usepackage{verbatim}
\usepackage{setspace}
\usepackage[abs]{overpic}
\usepackage{graphicx}
\usepackage{subfigure}
\usepackage{bbm, bm}
\usepackage{dsfont}

\newcommand{\mb}{\boldsymbol}
\renewcommand{\bm}{\boldsymbol}
\newcommand{\mc}{\mathcal}

\newcommand{\norm}[2]{\left\| #1 \right\|_{#2}}
\newcommand{\reals}{\mathbb{R}}
\newcommand{\<}{\langle}
\renewcommand{\>}{\rangle}
\newcommand{\innerprod}[2]{\left\< #1, #2 \right\>}
\newcommand{\set}[1]{\left\{ #1 \right\}}

\newcommand{\rank}{\mathrm{rank}}
\newcommand{\diag}[1]{\mathrm{diag}\left( #1 \right)}

\usepackage{multirow}
\usepackage{authblk}

\usepackage{array}
\newcolumntype{L}[1]{>{\raggedright\let\newline\\\arraybackslash\hspace{0pt}}m{#1}}
\newcolumntype{C}[1]{>{\centering\let\newline\\\arraybackslash\hspace{0pt}}m{#1}}
\newcolumntype{R}[1]{>{\raggedleft\let\newline\\\arraybackslash\hspace{0pt}}m{#1}}

\begin{document}
\mainmatter              
\title{Revisiting Skip-Gram Negative Sampling Model \\With Rectification}
\titlerunning{Revisiting SGNS Model}  
%
\author{Cun (Matthew) Mu, Guang Yang, Yan (John) Zheng}

\authorrunning{C. Mu, G. Yang, and Y. Zheng} 
%
%
\institute{Jet.com/WalmartLabs, Hoboken, NJ 07030, USA\\
\email{\{matthew.mu, guang, john\}@jet.com}}

\maketitle              

\begin{abstract}
We revisit skip-gram negative sampling (SGNS), one of the most popular neural-network based approaches to learning distributed word representation. We first point out the ambiguity issue undermining the SGNS model, in the sense that the word vectors can be entirely distorted without changing the objective value. To resolve the issue, we investigate intrinsic structures in solution that a good word embedding model should deliver. Motivated by this, we rectify the SGNS model with quadratic regularization, and show that this simple modification suffices to structure the solution in the desired manner. A theoretical justification is presented, which provides novel insights into quadratic regularization . Preliminary experiments are also conducted on Google's analytical reasoning task to support the modified SGNS model.
\end{abstract}
\section{Introduction}
Distributed word representations, a.k.a.  word embeddings, represent each word with a real-valued vector as an approximation to its linguistic meaning. Different from the traditional discrete and sparse one-hot encoding, such continuous and dense representations are shown to  better capture syntactic and semantic regularities in language, and have been successfully applied in various natural language processing tasks, such as document classification \cite{kim2014convolutional}, information retrieval \cite{grbovic2015context}\cite{nalisnick2016improving}, question answering \cite{iyyer2014neural}\cite{shih2016look}, named entity recognition \cite{sienvcnik2015adapting}\cite{lample2016neural}, and parsing \cite{socher2013parsing}.

One of the main approaches to learning distributed word representation is the neural-network based one (\cite{bengio2003neural}\cite{morin2005hierarchical}\cite{bengio2006neural}\cite{collobert2008unified}\cite{mnih2009scalable}\cite{collobert2011natural}\cite{le2011structured}\cite{baroni2014don}\cite{pennington2014glove}), in which word vectors are trained to maximize the likelihood of word-context occurrences observed from large text corpus (e.g., news collections, Wikipedia and Web Crawl) based on probabilistic models. In particular, a series of recent papers by Mikolov et al. \cite{mikolov2010recurrent}\cite{mikolov2013linguistic}\cite{mikolov2013efficient}\cite{mikolov2013distributed}\cite{mikolov2017advances} culminated in and popularized the skip-gram model with negative-sampling training scheme (a.k.a. the SGNS model), which together with its variants \cite{sun2016sparse}\cite{yang2017simple} is shown to achieve state-of-the-art results on a variety of linguistic tasks.


Despite the empirical success of the SGNS model,  in this paper, we will first point out an observation that the optimization problem introduced by the SGNS model is essentially an {\em ill-posed} one. In specific, we can easily distort the output solution without changing its objective value. To fix this issue, we investigate solution structures that a good word embedding model should deliver, and argue that a meaningful word embedding model should allow and only allow the ambiguities introduced by {\em orthogonal transformations}. Motivated by this goal, we rectify the SGNS model by appending quadratic regularization terms to the original objective of SGNS, and show this simple modification {\em suffices} in enforcing the solution to be structured in the desired manner. A theoretical justification is presented, which provides novel insights into quadratic regularization. Preliminary experiments are conducted to evaluate word vectors on Google's analytical reasoning task, which shows the modified SGNS model outperforms the original SGNS model in a consistent manner. 


\section{SGNS Model} \label{sec: review}
The SGNS model is essentially the skip-gram word neural embedding model introduced in \cite{mikolov2013efficient} trained using the negative-sampling procedure proposed in \cite{mikolov2013distributed}. In this section, we will briefly review the SGNS model together with its related notation. Although the SGNS model is initially proposed and described in the the language of neural network, we find the explanation provided by Goldberg and Levy \cite{goldberg2014word2vec} is more transparent and could better disclose the rationale behind the model. Therefore, in the following, we adopt their approach in formulating the SGNS model.

Let $\mc W$ be the word vocabulary of our interest with $n:= |\mc W|$. The training data $\mc D$, normally collected based on some text corpus, consists of word-context pairs $(w,c)\in \mc W \times \mc W$ in both positive and negative sense. For a word $w$, its positive context word $c$ is often sampled from the neighborhood centering around the locations where $w$ shows up in the text corpus, while its negative context word $c$ is normally sampled from $\mc W$ randomly according to certain predefined distribution \cite{levy2015improving}. For each word $w \in \mc W$, its \textit{center-word embedding} and 
\textit{context-word embedding} are assumed to exist and represented as $\mc U[w]$ and  $\mc V[w]$, where 
\begin{flalign}
\mc U: \mc W \to \reals^d \quad \mbox{and} \quad \mc V: \mc W \to \reals^d.
\end{flalign}
The center-word embedding $\mc U[\cdot]$ is normally outputted as word representation, which will be used either by itself or as an important ingredient in subsequent natural language processing and machine learning applications.

The SGNS model learns the embeddings by solving the following optimization problem,
\begin{flalign}\label{eqn: SGNS}
\max_{\mc U: \mc W \to \reals^d, \;\mc V: \mc W \to \reals^d} &\quad \sum_{(w,c)\in \mc D^+} \log \sigma(\mc U[w]^\top \mc V[c])  +  \sum_{(w,c)\in \mc D^-}  \log \sigma(-\mc U[w]^\top \mc V[c]),
\end{flalign}
where $\mc D^+$ and $\mc D^-$ denotes the positive and negative pairs in $\mc D$, and 
$\sigma(\cdot)$ denotes the usual sigmoid function, i.e. $\sigma(x) = 1/(1+\exp(-x))$. For simplicity, we denote the center-word embedding matrix $\mb U$ (resp. context-word embedding matrix $\mb V$) as the matrix in $\reals^{n \times d}$ whose row vectors are stacked by the center-word embeddings (resp. context-word embedding) of all words from the vocabulary. We will use $\bm u_i$, $\bm v_i \in \reals^d$ to denote the $i$-th row of $\mb U$ and $\mb V$. With a slight abuse of notation, we will also use interchangeably  $\mb U[w]$ and $\mc U[w]$,  $\mb V[w]$ and $\mc V[w]$, i.e.
\begin{flalign}
\mb U[w]:= \mc U[w] \quad \mbox{and} \quad \mb V[w]:=\mc V[w]
\end{flalign}
to represent the center-word and the context-word embeddings of the word $w \in \mc W$. Then clearly we can rewrite \eqref{eqn: SGNS} equivalently as a maximization problem over the matrices $\mb U$ and $\mb V$ in $\reals^{n \times d}$,
\begin{flalign}\label{eqn: SGNS-matrix}
\max_{\mb U, \mb V \in \reals^{n \times d}} &\quad \mc L(\mb U, \mb V) := \sum_{(w,c)\in \mc D^+}  \log \sigma(\mb U[w]^\top \mb V[c]) + \sum_{(w,c)\in \mc D^-} \log \sigma(-\mb U[w]^\top \mb V[c]).
\end{flalign}

The SGNS model models how words are interacted with their contexts, which is rooted deeply in the distributional hypothesis of Harris \cite{harris1954distributional}, stating that {\em words sharing similar contexts possess similar meanings}. Intuitively, the SGNS model attempts to find embeddings $\set{\mb U[w]}_{w \in \mc W}$ and $\set{\mb V[c]}_{c \in \mc W}$ in a way such that their inner-products are encouraged to be large  for good context pairs, but to be small for bad ones. Several insightful interpretations--e.g., implicit matrix factorization \cite{levy2014neural}, representation learning \cite{li2015word}, weighted logistic PCA \cite{landgraf2017word2vec}, to just name a few--have been further proposed to better understand the underlying principles of the model. However, as we will point out in the next section, the SGNS model is essentially an ill-posed problem from the perspective of optimization.

\section{Ambiguity in the SGNS Model}

\begin{figure*}[t]
	\begin{center}
		\includegraphics[width= \textwidth]{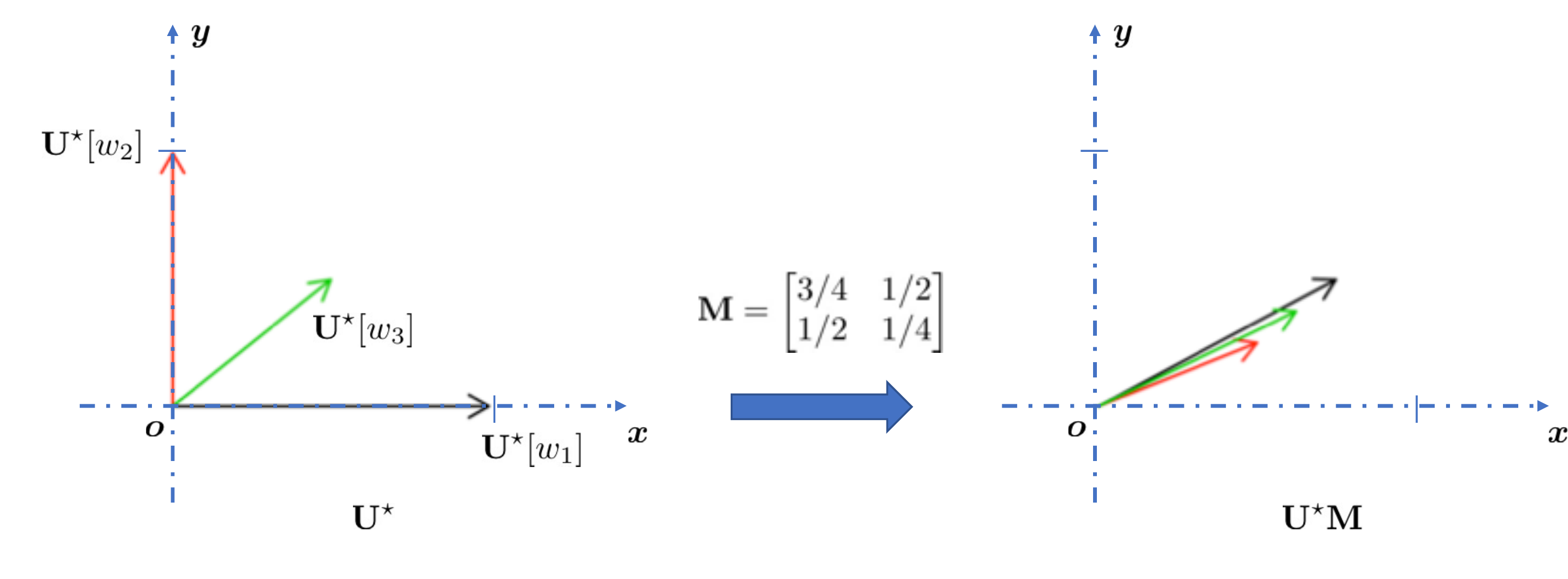}
		\caption{{\bf Illustration of Example \ref{ex: bad_case}.} \normalfont{Here we choose $\mb M$ as specified in \eqref{eqn: ex1_M} with $\varepsilon=1/4$. Although both embedding matrices $\mb U^\star$ and $\mb U^\star \mb M$ are solutions to the SGNS model, we can clearly observe that their word vectors are quite different in terms of encoded linguistic properties. 
		}} \label{fig: ex1}
	\end{center}
\end{figure*}

In this section, we will address a fundamental ambiguity issue undermining the SGNS model \eqref{eqn: SGNS-matrix}. Specifically, we will show that the solution from SGNS can be easily distorted without affecting the objective value.\footnote{In addition to the SGNS model, following the same logic, the fundamental ambiguity issue is shared by many other prevailing word embedding models (e.g., the CBOW model with negative sampling \cite{mikolov2013efficient}\cite{mikolov2013distributed} \cite{mikolov2017advances}, and the GloVe model \cite{pennington2014glove}.}

Suppose $(\mb U^\star, \mb V^\star)$ is one optimal solution to \eqref{eqn: SGNS-matrix}. Then for any invertible matrix $\mb M \in \reals^{d \times d}$, $(\mb U^\star \mb M, \mb V^\star \mb M^{-\top})$ is another optimal solution to SGNS as the objective value remains the same:
{
\begin{flalign}
&\mc L(\mb U^\star \mb M, \; \mb V^\star \mb M^{-\top})  \\
& =  \sum_{(w,c)\in \mc D^+}  \log \sigma\left(\innerprod{\mb M^\top \mb U^\star [w]}{\mb M^{-1} \mb V^\star [c]}\right) \nonumber \\
& \qquad \qquad \qquad \quad \quad \quad + \;\; \sum_{(w,c)\in \mc D^-} \log \sigma\left(-\innerprod{\mb M^\top \mb U^\star [w]}{\mb M^{-1} \mb V^\star [c]}\right) \nonumber \\
& =  \sum_{(w,c)\in \mc D^+}  \log \sigma(\innerprod{\mb U^\star [w]}{\mb V^\star [c]}) + \sum_{(w,c)\in \mc D^-} \log \sigma(-\innerprod{\mb U^\star [w]}{\mb V^\star [c]}) \nonumber \\
&  = \mc L(\mb U^\star, \mb V^\star).  \nonumber
\end{flalign}}
Therefore, there is an extremely large amount of freedom to manipulate $(\mb U^\star, \mb V^\star)$ without affecting the optimality, which could lead to entirely different embeddings in terms of encoded semantic and syntactic properties (i.e., vector lengths and angles).  To better understand the severity of this ambiguity, let us think about the following toy example. 

\begin{example}\label{ex: bad_case}
	Suppose we have $\mc W = \set{w_1, w_2, w_3}$, and
	\[
	\mb U^\star = 
	\begin{bmatrix}
	\mb U^\star[w_1]  \\
	\mb U^\star[w_2]  \\
	\mb U^\star[w_3]
	\end{bmatrix}
	=
	\begin{bmatrix}
	1 & 0  \\
	0 & 1  \\
	1/2& 1/2
	\end{bmatrix},
	\]
	whose row vectors are pretty spread out in $\reals^2$. However, by choosing
	\begin{flalign}\label{eqn: ex1_M}
	\mb M = 
	\begin{bmatrix}
	1/2+  \varepsilon & 1/2  \\
	1/2  & 1/2 -  \varepsilon
	\end{bmatrix},
	\end{flalign}
	where $ 0 \neq \varepsilon \in \reals,$ as argued above, $\mb U^\star \mb M$  is also an optimal solution to \eqref{eqn: SGNS-matrix} with
	\[
	\mb U^\star \mb M = 
	\begin{bmatrix}
	1 & 0  \\
	0 & 1  \\
	1/2& 1/2
	\end{bmatrix}
	\begin{bmatrix}
	1/2+  \varepsilon & 1/2  \\
	1/2  & 1/2 -  \varepsilon
	\end{bmatrix} 
	=
	\begin{bmatrix}
	1/2 +  \varepsilon &\quad 1/2 \\
	1/2   &\quad 1/2  -  \varepsilon \\
	1/2 + \varepsilon/2   & \quad 1/2  - \varepsilon/2
	\end{bmatrix},
	\]
	whose row vectors now become almost parallel as $\varepsilon$ approaches 0. 
\end{example}

To sum up, even though $\mb U^\star$ and $\mb U^\star \mb M$ have entirely different word representations in essence, the SGNS model makes no differentiation among them. In order to ensure intrinsic embeddings being learned, we have to avoid those $\mb M$'s that distort the linguistic properties of the word vectors. As the linguistic properties of the word vectors are mostly reflected by their {\em lengths} and {\em inner products}, we should allow and only allow linear transformations that preserve these quantities. For arbitrary $\bm u, \bm v \in \reals^d$, we are guaranteed to have $\norm{\mb M \bm u}{} = \norm{\mb M \bm u}{} $ and $\innerprod{\mb M \bm u}{\mb M \bm v} = \innerprod{\bm u}{\bm v}$ if and only if $\bm M \in \reals^{d \times d}$ is orthogonal, i.e., $\mb M^\top \mb M = \mb I$. So
the only {\em innocuous} ambiguities are the ones resulting from orthogonal transformation. Geometrically, this means that the rows of the embedding matrix $\mb U$ are transformed through {\em rotation} and {\em reflection}. Therefore, an ideal word embedding model should be expected in general to have unique optimal solutions up to orthogonal transformation, i.e., 

\centerline{
	\textit{$[*] \qquad $ \
		$(\mb U^\star \mb M, \mb V^\star \mb M^{-\top})$ 
		is optimal if and only if $\mb M$ is orthogonal. }
}
We will elaborate how we are able to achieve this in the next section.

\section{SGNS Model with Quadratic Regularization}
In this section, we will work towards the goal stated in $[*]$ by modifying the SGNS model.

Let us consider the extended SGNS model with regularization, 
\begin{flalign}\label{eqn: sgns+r}
\max_{\mb U, \mb V \in \reals^{n \times d}} \quad \mc L(\mb U, \mb V)  - \mc R(\mb U, \mb V),
\end{flalign}
where $\mc R: (\reals^{n \times d}, \reals^{n \times d}) \to \reals \cup \set{+ \infty}$ is some regularizer. The aim is to leverage the regularization term $\mc R$ to enforce the solution to be unique up to orthogonal transformation without (on the other hand) making the model too hard to be optimized. In the following, we will choose $\mc R$ to be a simple quadratic form, and show this slight modification is sufficient to achieve the goal stated in $[*]$ and thus  resolve the ambiguity issues undermining the SGNS model \eqref{eqn: SGNS}.

Consider the following SGNS model with quadratic regularization (named as the SGNS-qr model thereafter)
\begin{flalign}\label{eqn: SGNS-qr}
\max_{\mb U, \mb V \in \reals^{n \times d}} &\quad f(\mb U, \mb V) :=  \sum_{(w,c)\in \mc D^+}  \log \sigma(\mb U[w]^\top \mb V[c]) + \sum_{(w,c)\in \mc D^-} \log \sigma(-\mb U[w]^\top \mb V[c]) \nonumber\\
&\quad \qquad \qquad \qquad - \frac{\lambda}{2} \norm{\mb U}{F}^2  - \frac{\lambda}{2} \norm{\mb V}{F}^2,
\end{flalign}
where $\lambda>0$  is the regularization parameter and $\norm{\cdot}{F}$ denotes the matrix Frobenius norm. A similar model has been proposed in \cite{johnson2014logistic} in the context of collaborative filtering, which falls into the general framework of low-rank models \cite{udell2016generalized} with the logistic loss function and the quadratic regularization. The quadratic regularizer $\mc R(\mb U, \;\mb V) := \frac{\lambda}{2} \norm{\mb U}{F}^2 + \frac{\lambda}{2} \norm{\mb V}{F}^2$ explicitly encourages entries in both $\mb U$ and $\mb V$ to be small in magnitude, which (perhaps surprisingly) has the effect of penalizing the non-orthogonal transformation. We will state this novel insight regarding quadratic regularization in the following theorem.
\begin{theorem}\label{thm: main}
	Let $(\mb U^\star, \mb V^\star)$ be an optimal solution to \eqref{eqn: SGNS-qr}. Suppose $\mb U^\star$ and $\mb V^\star$ are both full rank. Then $(\hat{\mb U}, \hat{\mb V}) := (\mb U^\star \mb M, \mb V^\star \mb M^{-\top})$ is an optimal solution if and only if $\mb M$ is orthogonal.
\end{theorem}
\begin{proof}
	Let us first prove the {\em if} direction. Since $\mb M$ is orthogonal, 
	\begin{flalign}
	\norm{\mb U^\star}{F} = \norm{\mb U^\star \mb M}{F}, \quad \norm{\mb V^\star}{F} = \norm{\mb V^\star  \mb M}{F} =  \norm{\mb V^\star  \mb M^{-\top}}{F},
	\end{flalign}
	and therefore
	\begin{flalign}
	f(\mb U^\star, \mb V^\star) &= \mc L(\mb U^\star, \mb V^\star) + \frac{\lambda}{2} \norm{\mb U^\star}{F}^2  + \frac{\lambda}{2} \norm{\mb V^\star}{F}^2 \nonumber \\
	& = \mc L(\mb U^\star \mb M, \mb V^\star \mb M^{-\top}) + \frac{\lambda}{2} \norm{\mb U^\star  \mb M}{F}^2  + \frac{\lambda}{2} \norm{\mb V^\star \mb M^{-\top}}{F}^2 \nonumber \\&
	= f(\hat{\mb U}, \hat{\mb V}), \nonumber
	\end{flalign}
	which implies the optimality of $(\hat{\mb U}, \hat{\mb V})$.
	
	In the rest of the proof,  we will focus on the {\em only if} direction. 
	
	Let $\mb U \mb \Sigma \mb V^\top$ be the reduced singular value decomposition (SVD) \cite{trefethen1997numerical} of $\mb U^\star (\mb V^\star)^\top$, i.e., $\mb U^\star (\mb V^\star)^\top = \mb U \mb \Sigma \mb V^\top$  where $\mb U \in \reals^{n\times d}$ and $\mb V \in \reals^{n\times d}$ have orthonormal columns, and $\mb \Sigma = \diag{\sigma_1, \sigma_2, \ldots, \sigma_d}$ with $\sigma_1 \ge \sigma_2 \ge \cdots \ge \sigma_d >0.$  Here we write $\sigma_d >0$  since $\mb U^\star$ and $\mb V^\star$ are full rank, and by Sylvester inequality \cite{horn1990matrix}
	\begin{flalign*}
	d = \rank(\mb U^\star) + \rank(\mb V^\star) - d \le \rank(\mb U^\star (\mb V^\star)^\top)  \le  \min\{\rank(\mb U^\star), \rank(\mb V^\star)\}= d.
	\end{flalign*}
	Now we will first derive a upper bound for $f(\mb U^\star, \mb V^\star)$:
	\begin{flalign}\label{eqn: upp_bd}
	f(\mb U^\star, \mb V^\star)  &=  \mc L(\mb U^\star, \mb V^\star) -  \mc R(\mb U^\star, \mb V^\star) \nonumber  \\
	& = \mc L(\mb U^\star, \mb V^\star)  - \frac{\lambda}{2} \norm{\mb U^\star}{F}^2  - \frac{\lambda}{2} \norm{\mb V^\star}{F}^2 \nonumber \\
	& \le \mc L(\mb U^\star, \mb V^\star)  - \lambda \cdot \norm{\mb U^\star}{F}  \cdot \norm{\mb V^\star}{F} \nonumber \\
	& \le \mc L(\mb U^\star, \mb V^\star)  - \lambda \cdot \norm{\mb U^\top \mb U^\star}{F}  \cdot \norm{\mb V^\top \mb V^\star}{F} \nonumber \\
	& \le \mc L(\mb U^\star, \mb V^\star)  - \lambda \cdot \mbox{trace}(\mb U^\top \mb U^\star (\mb V^\star)^\top \mb V) \nonumber \\
	& = \mc L(\mb U^\star, \mb V^\star)  - \lambda \cdot \norm{\bm \sigma}{1},
	\end{flalign}
	where the third and the fifth lines uses Cauchy-Schwartz inequality, the fourth line holds as the operator norms $\norm{\mb U}{} \le 1$, $\norm{\mb V}{} \le 1$, and $\norm{\mb A\mb B}{F} \le \norm{\mb A}{} \norm{\mb B}{F} $ for any compatible matrices $\mb A$ and $\mb B$, and the last line follows directly from the definition of SVD.
	
	But on the other hand, we can also derive the following lower bound for $f(\mb U^\star, \mb V^\star)$:
	\begin{flalign} \label{eqn: low_bd}
	f(\mb U^\star, \mb V^\star) & \ge f(\mb U \mb \Sigma^{\frac{1}{2}}, \mb V \mb \Sigma^{\frac{1}{2}}) \nonumber \\
	& = \mc L(\mb U^\star, \mb V^\star) - \frac{\lambda}{2} \norm{\mb U \mb \Sigma^{\frac{1}{2}}}{F}^2  - \frac{\lambda}{2} \norm{\mb V \mb \Sigma^{\frac{1}{2}}}{F}^2  \nonumber \\
	& = \mc L(\mb U^\star, \mb V^\star) - \frac{\lambda}{2} \norm{\mb \Sigma^{\frac{1}{2}}}{F}^2  - \frac{\lambda}{2} \norm{\mb \Sigma^{\frac{1}{2}}}{F}^2  \nonumber \\
	& = \mc L(\mb U^\star, \mb V^\star)  - \lambda \norm{\bm \sigma}{1},
	\end{flalign}
	where $\mb \Sigma^{\frac{1}{2}}:=\diag{\sqrt \sigma_1, \sqrt \sigma_2, \ldots, \sqrt \sigma_d}$.
	
	Combining \eqref{eqn: upp_bd} and  \eqref{eqn: low_bd} , one can easily derive that 
	\begin{flalign}
	f(\mb U^\star, \mb V^\star) &= \mc L(\mb U^\star, \mb V^\star)  - \lambda \norm{\bm \sigma}{1}, \quad \mbox{and}  \\
	\frac{1}{2} \norm{\mb U^\star}{F}^ 2  +  \frac{1}{2} \norm{\mb V^\star}{F}^ 2 &= \norm{\mb U^\star}{F} \norm{\mb V^\star}{F}=\norm{\mb U^\star}{F}^ 2 = \norm{\mb V^\star}{F}^ 2 = \norm{\bm \sigma}{1}. \label{eqn: f_norm}
	\end{flalign}
	
	Now we are ready to show that $\mb U^\star = \mb U \mb \Sigma^{\frac{1}{2}} \mb Q$ for some orthogonal matrix $\mb Q \in \reals^{d \times d}.$ 
	
	As $\mb U \mb \Sigma \mb V^\top$ is the SVD of $\mb U^\star (\mb V^\star)^\top$, there exist full rank matrices $\mb S \in \reals^{d\times d}$ and $\mb T \in \reals^{d\times d}$ such that $\mb U^\star = \mb U \mb S$, $\mb V^\star = \mb V \mb T$ and $\mb S \mb T^\top = \mb \Sigma = \diag{\bm \sigma}.$  Then from \eqref{eqn: f_norm}, one  has 
	\begin{flalign} 
	\norm{\mb U^\star}{F} = \norm{\mb U \mb S}{F} = \norm{\mb S}{F} = \norm{\bm \sigma}{1}^{1/2}, 
	\label{eqn:S_F}\\
	\norm{\mb V^\star}{F} = \norm{\mb V \mb T}{F} = \norm{\mb T}{F} = \norm{\bm \sigma}{1}^{1/2}.
	\label{eqn:T_F}
	\end{flalign}
	
	Now let us write
	\begin{flalign}
	\mb X := 
	\begin{bmatrix}
	\mb S  \\
	\mb T
	\end{bmatrix}
	\begin{bmatrix}
	\mb S^\top 
	\mb T^\top
	\end{bmatrix}
	=
	\begin{bmatrix}
	\mb S \mb S^\top & \mb S \mb T^\top  \\
	\mb T \mb S^\top & \mb T \mb T^\top  \\
	\end{bmatrix}
	=
	\begin{bmatrix}
	\mb S \mb S^\top & \mb \Sigma  \\
	\mb \Sigma^\top  & \mb T \mb T^\top  \\
	\end{bmatrix} \succeq \mb 0.
	\end{flalign}
	Define
	\begin{flalign}\label{eqn: def_st_star}
	s^\star \in \arg\min_{i \in [d]} \set{(\mb S \mb S^\top)_{ii} - \sigma_i}  \quad \mbox{and} \quad
	t^\star  \in \arg\min_{i \in [d]} \set{(\mb T \mb T^\top)_{ii} - \sigma_i}.
	\end{flalign}
	Due to the facts that
	\begin{flalign}\label{eqn: st_sum}
	\sum_{ii} (\mb S \mb S^\top)_{ii}  = \norm{\mb S}{F}^2 = \sum_{i\in[d]} \sigma_i  \quad \mbox{and} \quad
	\sum_{ii} (\mb T \mb T^\top)_{ii}  = \norm{\mb T}{F}^2 = \sum_{i\in[d]} \sigma_i,
	\end{flalign}
	we must have 
	\begin{flalign}\label{eqn: st_star_neg}
	(\mb S \mb S^\top)_{s^\star s^\star} - \sigma_{s^\star} \le 0  \quad \mbox{and} \quad
	(\mb T \mb T^\top)_{t^\star t^\star} - \sigma_{t^\star} \le 0.
	\end{flalign}
	Since $\mb X$ is positive semidefinite \cite{horn1990matrix}, 
	\begin{flalign}
	(\bm e_{s^\star} - \bm e_{t^\star})^\top  \mb X (\bm e_{s^\star} - \bm e_{t^\star}) 
	= (\mb S \mb S^\top)_{s^\star s^\star} + (\mb T \mb T^\top)_{t^\star t^\star} - \sigma_{s^\star} - \sigma_{t^\star} \ge 0.
	\end{flalign}
	which together with \eqref{eqn: st_star_neg} leads to 
	\begin{flalign} \label{eqn: st_lower}
	(\mb S \mb S^\top)_{s^\star s^\star} = \sigma_{s^\star} \quad\mbox{and} \quad
	(\mb T \mb T^\top)_{t^\star t^\star} = \sigma_{t^\star}.
	\end{flalign}
	Combining \eqref{eqn: def_st_star} and \eqref{eqn: st_star_neg}, it can be easily verified that
	\begin{flalign}
	\diag{\mb S \mb S^\top}  = \bm \sigma = \diag{\mb T \mb T^\top}, 
	\end{flalign}
	which implies that for any $i \in [d]$, $\bm s_i$ and $\bm t_i$ (the $i$-th row of $\mb S$ and $\mb T$) satisfies $\norm{\bm s_i}{}^2 = \norm{\bm t_i}{}^2 = \sigma_i$. In addition, since $\mb S \mb T^\top = \mb \Sigma$, the inner-product $\innerprod{\bm s_i}{\bm t_i} = \sigma_i.$ Due to Cauchy-Schwartz inequality, $\bm s_i = \bm t_i$.  Therefore, $\mb S = \mb T$, $\mb \Sigma = \mb S \mb T^\top =  \mb S \mb S^\top =  \mb T \mb T^\top$. Then it can be easily verified that $\mb S = \mb T= \mb \Sigma^{\frac{1}{2}} \mb Q$ for some orthogonal matrix $\mb Q$. Therefore, we have proved that $\mb U^\star = \mb U \mb \Sigma^{\frac{1}{2}} \mb Q$ for some orthogonal matrix $\mb Q \in \reals^{d \times d}$. 
	
	Next, as $(\hat{\mb U}, \hat{\mb V})$ is also optimal and $\hat{\mb U}\hat{\mb V}^\top = \mb U^\star (\mb V^\star)^\top$, we can follow exactly the same argument to show that $\hat{\mb U} = \mb U \mb \Sigma^{\frac{1}{2}} \hat{\mb Q}$ for anther orthogonal matrix $\hat{\mb Q} \in \reals^{d \times d}.$ Therefore, in order to satisfy
	\begin{flalign}
	\hat{\mb U} = \mb U \mb \Sigma^{\frac{1}{2}} \hat{\mb Q} =  \underbrace{\mb U \mb \Sigma^{\frac{1}{2}} {\mb Q}}_{\mb U^\star}\mb M,
	\end{flalign}
	one must have $\mb M = \mb Q^\top \hat{\mb Q}$, which is also orthogonal. That completes our proof.
\end{proof}

Theorem \ref{thm: main} states that optimal solutions to \eqref{eqn: SGNS-qr} are not unique, but are essentially all equivalent in terms of their encoded  linguistic properties, as a result of the quadratic regularization removing all the adversarial ambiguities (e.g. the one described in Example \ref{ex: bad_case}) undermining the original SGNS model \eqref{eqn: SGNS-matrix}.  

\section{Experiment}

In this section, we will conduct some preliminary experiments to compare the SGNS model with our SGNS-qr model.

\paragraph{Algorithm.} We use the popular toolbox \textsf{word2vec} \cite{mikolov2013efficient}\cite{mikolov2013distributed} with its default parameter setting to solve the SGNS model, which leverages the standard {\em stochastic gradient method} (SGM) \cite{robbins1951stochastic}\cite{bertsekas2011incremental} to optimize the objective. 
We solve the SGNS-qr model by modifying the SGM in \textsf{word2vec} to accommodate the additional quadratic terms. 

\paragraph{Dataset.} We use a publicly accessible dataset Enwik9\footnote{http://mattmahoney.net/dc/textdata.html} as our text corpus, which contains about 128 million tokens collected from English Wikipedia articles. The vocabulary $\mc W$ is constructed by filtering out words that appear less than $200$ times. The positive and negative word-context pairs are generated in exactly the same manner with the one implemented in \textsf{word2vec} using its default setting. We adopt Google's analogy dataset to evaluate word  embeddings on analytical reasoning task. 

\paragraph{Evaluation.} In Google's analogy dataset,  $19,544$ questions are presented with the form ``$a$ is to $a^\star$ as $b$ is to $b^\star$'', where $b^\star$ is hidden and to be inferred from the whole vocabulary $\mc W$ based on the input $(a, a^\star, b).$ Among all these analogy questions, around half of them are syntactic ones (e.g., ``think is to thinking as code is to coding''), and the other half are semantic ones (e.g., ``man is to women as king is to queen'').   The questions are answered using the  \textsc{3CosMul} scheme \cite{levy2014linguistic}:
\begin{flalign}\label{eqn: 3CosMul}
\mc B^{\star} = \arg \max_{x \in \mc W/\set{a,a^\star, b}} \frac{\cos(\mc U[x], \mc U[a^\star])  \cdot  \cos(\mc U[x], \mc U[b]) }{\cos(\mc U[x], \mc U[a]) + \varepsilon}
\end{flalign}
where $\mc U: \mc W \to \reals^d$ is the word embedding to evaluate and $\varepsilon =$ 1e-3 is set to avoid zero-division. The performance is measured as the percentage of questions answered correctly, i.e., $ b^\star \in \mc B^
\star.$

\paragraph{Experiment result.}   We evaluate the SNGS model \eqref{eqn: SGNS-matrix} and the SGNS-qr model \eqref{eqn: SGNS-qr} with different choices of $\lambda$. The performance of each model is reported in Table \ref{Tab: experiment_result} in terms of the analytical reasoning accuracy. As presented in Table \ref{Tab: experiment_result}, within a wide and stable range of choices in $\lambda$, the SGNS-qr model outperforms the SGNS model ($\lambda=0$) in a consistent manner, and the improvement becomes more and more non-trivial  with the growth in the embedding dimension $d$.  To better visualize this, we plot in Figure \ref{fig: comparison} the prediction accuracies of the SGNS model and the SGNS-qr ($\lambda = 250$) model over $d$. As we can see clearly, the improvement rate rises from (nearly) $0\%$ to more than $3\%$ quickly as $d$ increases. This suggests that the ambiguity issue undermining the SGNS model becomes substantially more severe when the optimization problem \eqref{eqn: SGNS} is solved over larger ambient space. Remarkably, our simple rectification through quadratic regularization is capable of boosting the prediction accuracy by around $3\%$.\footnote{We note that similar empirical observation of the use of quadratic regularization being capable of improving the performance of the SGNS model has also been made by Vilnis and McCallum (2014) for a different NLP task:  word similarity task.}

\begin{center}
	\begin{table*}[h]
		{ 
			\hfill{}
			\begin{tabular}{|C{.75cm}|C{1.2cm}|C{1.2cm}|C{1.2cm}| C{1.2cm} | C{1.2cm} | C{1.2cm} | C{1.2cm} |}
				\hline
				\multirow{2}{0.75cm}{\centering $d$}& \multicolumn{7}{p{8.4 cm}|}{\centering $\lambda$} \\
				\cline{2-8} & \multicolumn{1}{c|}{0} & \multicolumn{1}{c|}{10} & \multicolumn{1}{c|}{50} & \multicolumn{1}{c|}{100} & \multicolumn{1}{c|}{250}  & \multicolumn{1}{c|}{500} & \multicolumn{1}{c|}{1000}\\ 
				\hline 
				\hline
				100  & 0.5642 & 0.5652 & 0.5666  & \textbf{0.5665} & 0.5645  & 0.5570 & 0.5397 \\
				200 & 0.6618 & 0.6617 & 0.6640 & 0.6656 & \textbf{0.6668} & 0.6605 & 0.6355  \\
				300 & 0.6768 & 0.6772  & 0.6798 & 0.6848 & \textbf{0.6909}  & 0.6869  & 0.6593 \\
				400 & 0.6851 & 0.6860 &  0.6902 & 0.6938 &  \textbf{0.7005} & 0.6952  & 0.6658 \\
				500 & 0.6909 & 0.6920 & 0.6947 & 0.6971 & \textbf{0.7035} & 0.6965  & 0.6554  \\
				600 & 0.6755 & 0.6763 & 0.6825 & 0.6888 & \textbf{0.6973} & 0.6926  & 0.6508  \\
				700 & 0.6781 & 0.6798 & 0.6835 & 0.6885 & \textbf{0.6981} & 0.6901  & 0.6399  \\
				800 & 0.6736 & 0.6744 & 0.6808 & 0.6848 & \textbf{0.6926} & 0.6860  & 0.6328  \\
				900 & 0.6713 & 0.6731 & 0.6785 & 0.6818 & \textbf{0.6903} & 0.6829  & 0.6275  \\
				1000 & 0.6622 & 0.6631 & 0.6689 & 0.6738 & \textbf{0.6820} & 0.6716  & 0.6181  \\
				\hline
		\end{tabular}}
		\hfill{}
		\vspace{2mm}
		\caption{\textbf{Evaluation of SGNS ($\lambda = 0$) and SGNS-qr models on Google's analytical reasoning task.} 
		}
		\label{Tab: experiment_result}
	\end{table*}
\end{center}

\begin{figure*}[h]
	\begin{center}
		\includegraphics[width=0.85 \textwidth]{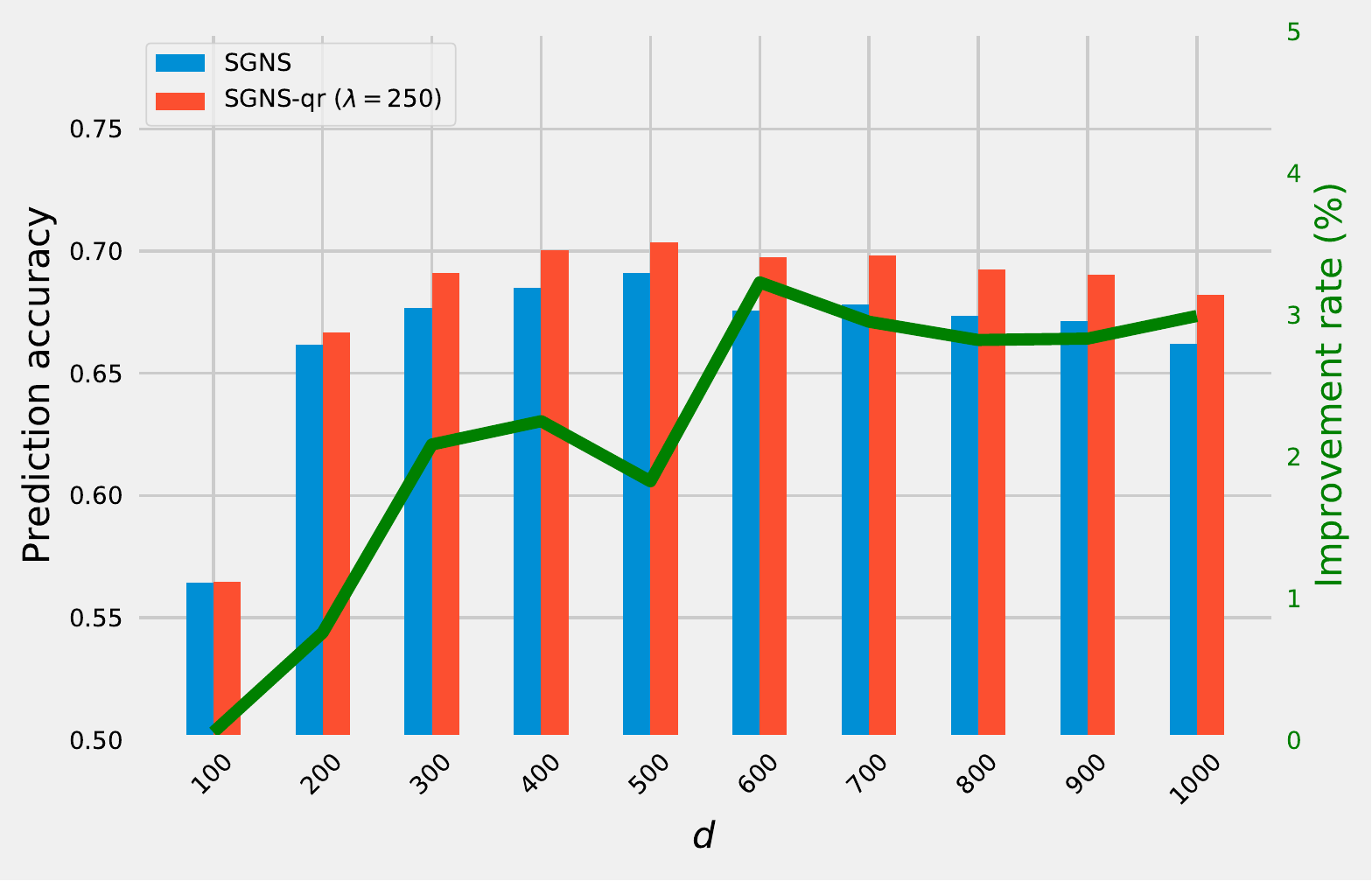}
		\caption{{\bf Comparison between SGNS and SGNS-qr model on Google's analytical reasoning task.} \normalfont{When the embedding dimension is small (e.g., $d = 100$), the SGNS-qr model is almost on a par with the SGNS model. But when $d$ becomes larger, the SGNS-qr model soon surpasses the SGNS model. Remarkably, the improvement is increasingly enlarged with the growth in $d$ and culminates in a boost of around $3\%$ in prediction accuracy.
		}} \label{fig: comparison}
	\end{center}
\end{figure*}

\section{Future Work}
In this paper, we rectify the SGNS model with quadratic regularization, and prove that this simple modification cures ambiguity issues undermining the SGNS model. Formulating the appropriate optimization to solve is an important but first step towards learning word vectors in a robust and efficient manner. We believe a (possibly) larger gain from this rectification comes from the perspective of optimization algorithm. Numerical methods, which perform poorly on machine learning tasks related with the SGNS model, might be solely due to the ill-posedness of the model rather than the inefficacies of the algorithms. In the future, we will tailor some recently designed numerical optimization methods (e.g., \cite{reddi2016stochastic}\cite{wang2017stochastic}\cite{goldfarb2017using}\cite{fonarev2017riemannian}\cite{reddi2018convergence}\cite{chen2018stochastic}) beyond SGM to solve our SGNS-qr model. Another interesting research direction is to resolve the ambiguity issue by leveraging higher-order relations among words and estimating underlying word embeddings through tensor decompositions \cite{kolda2009tensor}\cite{anandkumar2014tensor}\cite{mu2015successive}\cite{mu2017greedy}\cite{bailey2017word}\cite{frandsen2018understanding}.

\bibliographystyle{unsrt}
\bibliography{word2vec}

\end{document}